# An upper bound on the silhouette evaluation metric for clustering

Hugo Sträng [a,*], Tai Dinh [b,*]

[a] Independent researcher, Stockholm, Sweden
[b] Institute of Digital Technology, Thu Dau Mot University, Ho Chi Minh city, Vietnam

ABSTRACT

The silhouette coefficient quantifies, for each observation, the balance between within-cluster cohesion and between-cluster separation, taking values in the range $[-1, 1]$. The average silhouette width (ASW) is a widely used internal measure of clustering quality, with higher values indicating more cohesive and well-separated clusters. However, the dataset-specific maximum of ASW is typically unknown, and the standard upper limit of 1 is rarely attainable. In this work, we derive for each data point a sharp upper bound on its silhouette width and aggregate these to obtain a canonical upper bound for the ASW. This bound—often substantially below 1—enhances the interpretability of empirical ASW values by providing guidance on how close a given clustering result is to the best possible outcome for that dataset. We evaluate the usefulness of the upper bound on a variety of datasets and conclude that it can meaningfully enrich cluster quality evaluation; however, its practical relevance depends on the specific dataset. Finally, we extend the framework to establish an upper bound for the macro-averaged silhouette.

## 1. Introduction

Cluster analysis is a fundamental tool in data science for discovering structure in unlabeled data [1]. Its applications span the natural and social sciences, engineering, education, economics, and the health sciences [2]. Because ground truth labels are rarely available, evaluating clustering quality is challenging, so researchers rely on internal validation indices that summarize within-cluster compactness and between-cluster separation [3].

Among the internal cluster validation indices, the silhouette width (also known as the silhouette coefficient, silhouette score, or silhouette index) is one of the most widely used measures. It quantifies how well each point fits within its assigned cluster relative to its nearest neighboring cluster [4]. The average of this measure across all data points is often referred to as the average silhouette width (ASW). Silhouette scores close to 1 signal well-separated, compact clusters; values near 0 typically occur for points that lie between two clusters; a negative value indicates that a point is on average closer to its nearest neighboring cluster than to its own and is therefore commonly used to flag a potential misplacement. These interpretations are most reliable when clusters are roughly isotropic and of comparable scale. When cluster diameters differ substantially, silhouette values can be misleading, a limitation we return to in Section 6.

Extensive work has examined and refined ASW. Batool et al. [5] proved that it satisfies key axioms for clustering quality measures such as scale invariance and consistency. More recently, Pavlopoulos et al. [6] examined alternative aggregation strategies that improve ASW's robustness to cluster size imbalance. The so-called *medoid-based* silhouette is a simple approximation of the silhouette coefficient that replaces average distances to clusters with distances to cluster medoids, making computations cheaper [7]. The medoid of a cluster is the cluster member with the smallest total distance to all other members [8].

Although ASW is widely used, its raw values are difficult to interpret. Intrinsic data characteristics impose strict ceilings on attainable ASW scores; for example, when clusters overlap or deviate from convex shapes, even the best partition may yield a low ASW. As a result, observing a low ASW value alone does not reveal whether this reflects poor clustering performance or an inherent limitation of the data itself. For any dataset with more than a few samples, the true maximum ASW is unknown, which makes it difficult to assess clustering performance in relative terms. Relating an empirical ASW to a data-dependent upper bound could mitigate this problem by indicating how close the given solution is to the best possible outcome.

Against this background, we pose the following research question.

> *Question: Given a dissimilarity matrix, can we efficiently compute an upper bound for the* ASW *that provides meaningful guidance on how close an empirical clustering is to the global* ASW *maximum?*

* Corresponding authors.
*E-mail addresses:* hugo.strang2@gmail.com (H. Sträng), it08duytai@gmail.com, taidd@tdmu.edu.vn (T. Dinh).

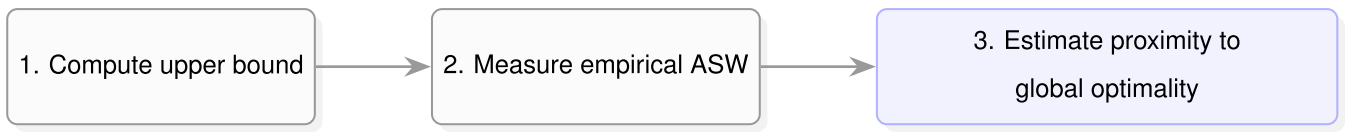

**Fig. 1.** Illustrative workflow for using the upper bound as a diagnostic tool. The process involves (1) computing the ASW upper bound from the given dissimilarity matrix, (2) measuring the ASW achieved by the chosen clustering algorithm, and (3) assessing the proximity to global optimality by analyzing the gap between these two values.

In connection with this question, our main contributions are as follows:

- We introduce a novel data-dependent upper bound on the ASW, computable in $O(n^2 \log n)$ time, to yield a global ceiling that no clustering of the given dissimilarity matrix can surpass.
- We evaluate the usefulness of the proposed upper bound on synthetic datasets and a diverse suite of real-world datasets across multiple dissimilarity measures.
- All datasets, preprocessing scripts, bound computation routines, and experiment notebooks are publicly released in a GitHub repository[1] and PyPI[2] to ensure full reproducibility and to encourage further research and practical adoption.

In theory, the proposed upper bound can provide important context for interpretation: if analysis shows that no partition can exceed, say, an ASW of 0.30, then an empirical result of 0.29 can be understood as being close to the best achievable value for the given data. Such information helps practitioners assess how much room for improvement remains and avoid expending effort in search of gains that are provably limited by the structure of the data.

The biggest obstacle in silhouette optimization is the combinatorial explosion of the solution space as the number of samples increases. Even for moderately sized datasets, the number of possible clusterings is astronomical, rendering optimization through exhaustive search with respect to any internal criterion computationally prohibitive. Batool et al. [5] addressed this limitation by proposing two heuristic algorithms that search for local maxima of ASW, noting that global optimization is tractable only for the smallest datasets. Our contribution complements their approach by providing a data-aware benchmark against which the quality of such locally optimal solutions can be evaluated. Fig. 1 illustrates the intended workflow.

To summarize: this paper's main objective is to propose and assess a method for computing an upper bound on the ASW for a given dissimilarity matrix. We believe this contribution is important because such a data-dependent upper bound could be useful in clustering quality evaluations.

The rest of this paper is organized as follows. Section 2 reviews related works. Section 3 introduces preliminaries. Section 4 details the proposed method. Section 5 presents and discusses the experimental results. Finally, Section 6 summarizes the findings and outlines directions for future research.

Lastly, Table 1 summarizes the notation used throughout the paper.

## 2. Related work

The silhouette coefficient [4] is a distance-based internal validity index that quantifies how well each object fits into its assigned cluster relative to its nearest alternative. Fig. 2 illustrates the construction. For a point $x_i$ in cluster $C_I$, we measure within-cluster cohesion by $a(x_i)$—its average dissimilarity to the other members of $C_I$—and nearest-cluster separation by $b(x_i)$—its average dissimilarity to the closest neighboring cluster. These quantities yield an individual silhouette score $s(x_i) \in [-1, 1]$, which is large when separation dominates cohesion, near zero around decision boundaries, and negative when $x_i$ is, on average, closer to another cluster than to its own. Aggregating over points gives the ASW, a dataset-level summary of partition quality.

[1] https://github.com/hugo-strang/silhouette-upper-bound
[2] https://pypi.org/project/silhouette-upper-bound/

**Table 1**
Summary of notation used throughout the paper.

| Symbol | Description |
|---|---|
| $n$ | Number of objects in the dataset |
| $i, j$ | Indices of objects, in $\{1, \ldots, n\}$ |
| $\Delta \in \mathbb{R}^{n\times n}$ | Dissimilarity matrix with entries $\Delta_{ij}$ |
| $\hat{\Delta} \in \mathbb{R}^{n\times(n-1)}$ | Row-wise sorted off-diagonal dissimilarities |
| $K$ | Number of clusters |
| $C_K = \{C_I\}_{I=1}^K$ | A $K$-clustering of $\{1, \ldots, n\}$ |
| $C_I$ | Cluster $I$; $\lvert C_I\rvert$ denotes its size |
| $C^*$ | Silhouette-optimal clustering |
| $a(i)$ | Average dissimilarity from object $i$ to members of its own cluster |
| $b(i)$ | Minimum average dissimilarity from $i$ to points in any other cluster |
| $s(i \mid C_K, \Delta)$ | Silhouette width of object $i$ |
| $\mathrm{ASW}(C_K, \Delta)$ | Average silhouette width |
| $S_{\mathrm{macro}}(C_K, \Delta)$ | Macro-averaged silhouette |
| $q(i, \Delta, k)$ | $k$-quotient of object $i$ used in the upper bound construction |
| $m$ | Minimum allowed cluster size in constrained solution spaces |
| $f(i, \Delta)$, $f_m(i, \Delta)$ | Minima of $q(i, \Delta, k)$ over $k$ (unconstrained and with $\min_I \lvert C_I\rvert \geq m$) |
| $\mathrm{UB}(\Delta)$ | Global upper bound on ASW |
| $\mathrm{UB}_m(\Delta)$ | Upper bound on ASW with minimum cluster size $m$ |

In practice, ASW is often used for model selection by choosing the number of clusters $K$ that maximizes it [9–11]. Other commonly used internal validation indices include the Davies–Bouldin index [12] and the Calinski–Harabasz index [13], which also evaluate compactness and separation but rely on alternative formulations of intra- and inter-cluster dispersion.

In this work, we focus on the silhouette due to its interpretability and widespread adoption across both research and applied domains. To account for cluster size imbalance, we additionally employ the *macro-averaged* silhouette [6], defined as the unweighted mean of per-cluster silhouette averages. This variant ensures that smaller clusters contribute equally to the overall assessment, providing a more balanced view of clustering quality.

The existing literature focuses mainly on (1) comparing the silhouette score with other clustering quality indices [17], (2) proposing closely related alternatives that are preferable in certain contexts [6,15], or (3) designing algorithms for silhouette-optimization [5,7,15,16]. To the best of our knowledge, these works do not provide an efficiently computable, data-dependent upper bound on the best achievable ASW for a given dataset. Our contribution complements these lines of work by deriving such a bound directly from the pairwise dissimilarities, enabling empirical silhouette scores to be interpreted relative to a dataset-specific ceiling and providing a principled benchmark for silhouette-optimization heuristics.

*Silhouette-optimization.* PAMSIL (Partitioning Around Medoids based on SILhouette) is a modification of the classical PAM algorithm [14], designed to directly optimize the ASW [15]. Unlike $K$-means or standard $K$-medoids, which minimize within-cluster dissimilarities, PAMSIL explicitly seeks cluster configurations that maximize ASW. It achieves this by modifying the SWAP phase of PAM: in each iteration, the algorithm performs the swap that yields the largest improvement in ASW and terminates once no further improvement is possible, ensuring convergence to a local optimum. An efficient implementation of PAMSIL is available in the open-source `kmedoids` Python package [16]. Moreover, Fatima Batool and Christian Hennig have proposed alternative algorithms that find local ASW optima [5].

In practice, PAMSIL can be computationally demanding, particularly for larger datasets. For this reason, accelerated variants have been developed. The `kmedoids` package includes FastMSC and DynMSC, which are both methods for optimizing the medoid-based version of the silhouette. DynMSC automatically explores different values of $K$ and returns the clustering with the highest average medoid silhouette.

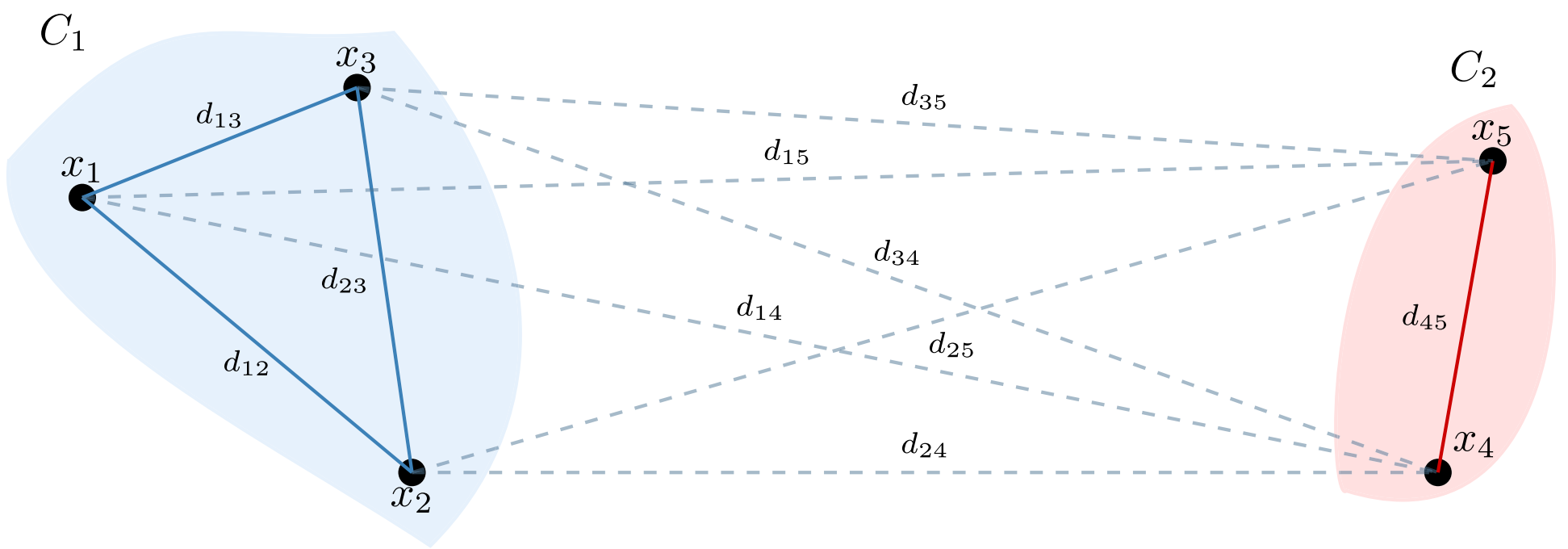

**Silhouette coefficient for $x_i$:**

$a(x_i) = \frac{1}{|C(x_i)|-1} \sum_{\substack{x_j \in C(x_i) \\ j \neq i}} d_{ij}$ (intra-cluster distance)

$b(x_i) = \min_{C \neq C(x_i)} \frac{1}{|C|} \sum_{x_j \in C} d_{ij}$ (inter-cluster distance to the nearest cluster)

$s(x_i) = \frac{b(x_i) - a(x_i)}{\max\{a(x_i), b(x_i)\}}$

**Average silhouette:** $\bar{s} = \frac{1}{n} \sum_{i=1}^{n} s(x_i)$

**Fig. 2.** Illustration of the silhouette coefficient calculation.

*External validation indices.* When ground truth class labels are available, external validation measures provide a way to quantify how well a clustering agrees with a known reference partition. Two of the most commonly used indices are the adjusted Rand index (ARI) [18] and the adjusted mutual information (AMI) [19]. Both adjust their underlying similarity measures for chance agreement, ensuring that random labelings yield expected scores near zero.

The ARI extends the classical Rand index, which compares all pairs of samples and counts how consistently they are assigned to the same or different clusters in both partitions. Its normalization accounts for random overlap and produces values in $[-1, 1]$ with 1 indicating identical partitions. The AMI applies a similar correction to the mutual information (MI) between partitions, normalizing it by the expected MI under random assignments and bounding the result between 0 and 1.

Both indices are symmetric and widely used in benchmarking studies, where they enable consistent comparison of clustering algorithms across datasets with known class structure.

## 3. Preliminaries

Let $n \in \mathbb{N}$ with $n \geq 4$. We consider a dataset consisting of $n$ objects indexed by the set $\{1, \dots, n\}$. Each pair of objects is associated with a measure of dissimilarity (or distance), and all pairwise dissimilarities are collected in the dataset's dissimilarity matrix.

**Definition 1** (Dissimilarity matrix). A matrix $\Delta \in \mathbb{R}^{n \times n}$ is a **dissimilarity matrix** if

1. $\Delta_{ij} \geq 0$ for all $i, j$ ($\Delta$ is non-negative),
2. $\Delta_{ii} = 0$ for all $i$ ($\Delta$ has zero diagonal), and
3. $\Delta_{ij} = \Delta_{ji}$ for all $i, j$ ($\Delta$ is symmetric).

The triangle inequality is not required. We also exclude matrices containing a row of all zeros.

Now, let us formally define the concept of a clustering of the dataset $\{1, \dots, n\}$.

**Definition 2** (Clustering). Let $K$ be an integer between 1 and $n$. A **K-clustering** of $\{i\}_{i=1}^{n}$ is a set $\mathcal{C}_K = \{C_I\}_{I=1}^{K}$ such that $C_I \subseteq \{1, \dots, n\}$ for all $I$, and

1. $\bigcup_{I=1}^{K} C_I = \{1, \dots, n\}$,
2. $C_I \cap C_J = \emptyset$ for $I \neq J$, and
3. $C_I \neq \emptyset$ for all $I \in \{1, \dots, K\}$.

**Remark.** The case $K = 1$ corresponds to the trivial clustering $\mathcal{C}_1 = \{\{1, \dots, n\}\}$, while $K = n$ corresponds to the trivial partition into singletons.

In the continuation of the article, we only consider $K > 1$ for convenience. Although we emphasize that the method proposed in Section 4 does not require such an assumption, and can be applied to datasets with a single cluster as well.

The total number of clusterings of an $n$-element set grows according to the Bell numbers $B_n$. For instance, the number of possible partitions of only 26 objects is $B_{26} = 49631246523618756274$.

We will now define the silhouette width, which was introduced in 1987 by P.J. Rousseeuw [4]. We employ similar notation.

**Definition 3** (Silhouette). Let $\mathcal{C}_K$ be a $K$-clustering of the dataset $\{1, \dots, n\}$ with dissimilarity matrix $\Delta$. The **silhouette width** for data point $i \in C_I$ is defined as

$$s(i \mid \mathcal{C}_K, \Delta) \begin{cases} \dfrac{b(i) - a(i)}{\max\{a(i), b(i)\}} & \text{if } |C_I| > 1 \wedge \max\{a(i), b(i)\} > 0, \\ 0 & \text{if } \big(|C_I| > 1 \wedge \max\{a(i), b(i)\} = 0\big) \vee |C_I| = 1 \end{cases}$$

where

$$a(i) := \frac{1}{|C_I| - 1} \sum_{j \in C_I \setminus \{i\}} \Delta_{ij}, \quad \text{and} \quad b(i) := \min_{J \neq I} \left( \frac{1}{|C_J|} \sum_{j \in C_J} \Delta_{ij} \right).$$

The **average silhouette width** (ASW) is defined as

$$\mathrm{ASW}(\mathcal{C}_K, \Delta) := \frac{1}{n} \sum_{i=1}^{n} s(i \mid \mathcal{C}_K, \Delta).$$

Note that the silhouette width of $i$ may be written as

$$s(i|\mathcal{C}_K, \Delta) = \begin{cases} 1 - a(i)/b(i) \text{ if } a(i) < b(i), \\ 0 \text{ if } a(i) = b(i), \\ b(i)/a(i) - 1 \text{ if } a(i) > b(i). \end{cases}$$

Consequently, $-1 \leq s(i \mid \mathcal{C}_K, \Delta) \leq 1$.

Individual silhouette widths provide a useful graphical diagnostic, but for large $n$ it is customary to summarize the clustering quality by ASW [7]. A common heuristic is to run a clustering algorithm for several choices of $K$ and select the value that maximizes ASW. When even that maximum is low, it is possible that the limitation stems from the structure of the dissimilarity data. This raises the natural question: given a dissimilarity matrix, what is the highest possible ASW?

In the remainder of the paper, we let $C^*$ denote a silhouette-optimal clustering. That is, if $C_K$ is a $K$-clustering of $\Delta$ then $\text{ASW}(C_K, \Delta) \leq \text{ASW}(C^*, \Delta)$.

## 4. A data-dependent upper bound on the silhouette

We propose a data-dependent ceiling on the ASW that can be computed before any clustering is performed. For each data point, we derive an upper bound on its silhouette width. Averaging these per-point ceilings produces a single, dataset-level upper bound on the ASW. The bound is computed from the dissimilarity matrix of the dataset in $O(n^2 \log n)$ time, and the proposed algorithm supports optional minimum cluster size constraints to further tighten the bound.

Consider an arbitrary clustering $C_K$ of the dataset $\{1, \dots, n\}$ with dissimilarites $\Delta$. We will now formalize the upper bound on the silhouette width for point $i$, i.e. $s(i|C_K, \Delta)$. For our purposes, it will be convenient to work with the off-diagonal entries of $\Delta$, sorted row-wise in ascending order. We denote this transformed dissimilarity matrix by $\hat{\Delta}$ ($\hat{\Delta} \in \mathbb{R}^{n \times n-1}$). In other words, we have for every $i$ that $\{\Delta_{ij}\}_{j=1}^{n} \setminus \{\Delta_{ii}\} = \{\hat{\Delta}_{ij}\}_{j=1}^{n-1}$ and $\hat{\Delta}_{i1} \leq \hat{\Delta}_{i2} \leq \dots \leq \hat{\Delta}_{i(n-1)}$. Constructing $\hat{\Delta}$ requires sorting each of the $n$ rows of $\Delta$, which costs $\mathcal{O}(n^2 \log n)$ time in total.

For data point $i$, let $|C_I|$ denote the size of its cluster in $C_K$. The average distance between $i$ and the remaining data points in the same cluster is not less than the average distance between $i$ and the $|C_I| - 1$ data points closest to $i$. Similarly, the average distance between $i$ and the $n - |C_I|$ points farthest from $i$ is not smaller than the average distance between $i$ and the data points in the neighboring cluster closest to $i$. These two observations are combined to construct a value that is guaranteed to be greater than or equal to $s(i|C_K, \Delta)$. We formalize this reasoning through the $k$-quotient of point $i$.

**Definition 4.** Let $k \in \mathbb{N}$ with $1 \leq k \leq n-1$. The *$k$-quotient* of data point $i$ is

$$q(i, \Delta, k) := \begin{cases} \dfrac{n-k}{k-1} \cdot \dfrac{\sum_{j=1}^{k-1} \hat{\Delta}_{ij}}{\sum_{j=k}^{n-1} \hat{\Delta}_{ij}} & \text{if } k > 1, \\ 1 & \text{if } k = 1. \end{cases}$$

Recall the following basic fact.

**Lemma 1.** *Let $A = \{a_1, \dots, a_N\}$ be a collection of $N$ real values, and consider a partition of $A$ into subsets $A_1, \dots, A_M$. Then*

$$\min_L \frac{1}{|A_L|} \sum_{a \in A_L} a \leq \frac{1}{|A|} \sum_{a \in A} a.$$

**Proof.** Let $\mu = \frac{1}{|A|} \sum_{a \in A} a$ and $\mu_L = \frac{1}{|A_L|} \sum_{a \in A_L} a$. We have $\mu = \sum_{L=1}^{M} \frac{|A_L|}{|A|} \mu_L$, and since $\sum |A_L| = |A|$, the statement holds. □

**Proposition 1.** *Let $C_K$ be a $K$-clustering of the dataset $\{i\}_{i=1}^{n}$ with dissimilarites $\Delta$. Let $i$ represent any data point and let $C_I \in C_K$ denote the cluster that contains $i$. Furthermore, let*

$$f(i, \Delta) := \min_{1 \leq k \leq n-1} \{q(i, \Delta, k)\}$$

*denote a minimal $k$-quotient. Then*

$$s(i|C_K, \Delta) \leq 1 - f(i, \Delta).$$

**Proof.** Suppose $q(i, \Delta, 2) > 1$. Then $(n-2)\hat{\Delta}_{i1} > \sum_{j=2}^{n-1} \hat{\Delta}_{ij}$, which contradicts $\hat{\Delta}_{i1} \leq \hat{\Delta}_{ij}$ for $j \geq 2$. Hence, $0 \leq f(i, \Delta) \leq 1$. This proves the statement if $|C_I| = 1$ and for cases $a(i) \geq b(i)$ when $|C_I| > 1$. Consider $q(i, \Delta, |C_I|)$. It is clear that $\sum_{j=1}^{|C_I|-1} \hat{\Delta}_{ij} / (|C_I| - 1) \leq a(i)$. It remains to argue that $\sum_{j=|C_I|}^{n-1} \hat{\Delta}_{ij} / (n - |C_I|) \geq b(i)$. By Lemma 1, we have

$$\frac{1}{n - |C_I|} \sum_{j=|C_I|}^{n-1} \hat{\Delta}_{ij} \geq \frac{1}{n - |C_I|} \sum_{j \in \bigcup_{J \neq I} C_J} \Delta_{ij}$$

$$\geq \min_{J \neq I} \left( \frac{1}{|C_J|} \sum_{j \in C_J} \Delta_{ij} \right) = b(i).$$

Thus, $f(i, \Delta) \leq q(i, \Delta, |C_I|) \leq a(i)/b(i)$ and $1 - f(i, \Delta) \geq 1 - a(i)/b(i)$. □

Taking the mean of all silhouette width upper bounds from Proposition 1 results in the following inequality, which is the main observation of this paper:

$$\text{ASW}(C^*, \Delta) \leq 1 - \frac{1}{n} \sum_i f(i, \Delta). \quad (1)$$

We remark that the right-hand side in (1) depends only on the underlying data $\Delta$. Crucially, we note that $1 - \frac{1}{n} \sum_i f(i, \Delta) \leq 1$.

We stress that the right-hand side in (1) is not, in general, a sharp upper bound on the ASW. However, for a single data point $i$, the quantity $1 - f(i, \Delta)$ is a sharp upper bound on its individual silhouette width in the sense that the individual silhouette width of the data point can never exceed the bound, and there exists a clustering in which the bound is achieved exactly. Indeed, considering the 2-clustering where $i$ is grouped with its $k-1$ nearest neighbors and all remaining points are grouped into a single separate cluster yields $s(i|C_K, \Delta) = 1 - f(i, \Delta)$.

Imposing the narrower definition

$$f_m(i, \Delta) := \min_{m \leq k \leq n-m} \{q(i, \Delta, k)\}$$

for some $m \in \{1, \dots, \lfloor n/2 \rfloor\}$ is equivalent to constraining the solution space to clusterings having $\min_I |C_I| \geq m$. By letting $C_m^*$ denote a silhouette-optimal clustering within this solution space, we have $\text{ASW}(C_m^*, \Delta) \leq 1 - \frac{1}{n} \sum_i f_m(i, \Delta)$.

In real-world applications, analysts sometimes need to enforce minimum cluster size constraints to avoid too small clusters [20,21]. From this perspective, the quantity $1 - \frac{1}{n} \sum_i f_m(i, \Delta)$ does not aim to approximate the unconstrained upper bound, but rather provides an upper bound tailored to the family of clusterings that a practitioner is willing to consider. This allows for a meaningful benchmark in settings where small clusters are inadmissible, and offers flexibility in situations where minimum cluster sizes are part of the problem specification.

### 4.1. Algorithm

**Algorithm 1:** Silhouette upper bound (pointwise).

**Input:** $\Delta \in \mathbb{R}^{n \times n}$; // Dissimilarity matrix
**Input:** $m \in \mathbb{N}$ ; // Smallest allowed cluster size, $1 \leq m \leq \lfloor n/2 \rfloor$
**Output:** $S \in \mathbb{R}^n$; // Vector of point-wise upper bounds

Initialize $S \leftarrow [\,]$ ;
Compute $\hat{\Delta}$ from $\Delta$;
**for** $i = 1, \dots, n$ **do**
  $y \leftarrow \sum_{j=m}^{n-1} \hat{\Delta}_{i,j}$ ;
  $x \leftarrow 0$ ;
  $q \leftarrow 1$ ; // Initialize $k$-quotient
  **if** $m > 1$ **then**
    $x \leftarrow \sum_{j=1}^{m-1} \hat{\Delta}_{i,j}$;
    $q \leftarrow \frac{x}{m-1} / \frac{y}{n-m}$;
  **for** $k = m+1, \dots, n-m$ **do**
    $x \leftarrow x + \hat{\Delta}_{i,k-1}$;
    $y \leftarrow y - \hat{\Delta}_{i,k-1}$;
    $q_{\text{candidate}} \leftarrow \dfrac{x}{k-1} / \dfrac{y}{n-k}$;
    **if** $q_{\text{candidate}} < q$ **then**
      $q \leftarrow q_{\text{candidate}}$ ; // Update minimal $k$-quotient
  Append $1 - q$ to $S$;
**return** $S$

Algorithm 1 computes for each data point $i$ the quantity $1 - f_m(i, \Delta)$, which are sharp upper bounds of the silhouette widths in the solution space $\min_I |C_I| \geq m$. The vector $S$ collects these values across all points, and its average equals the dataset-level upper bound. As we shall see in the Experiments section, the data point-level upper bounds provide a nice complement to standard silhouette visualizations, where the achieved point-wise silhouette widths are typically displayed next to each other in a bar graph.

*Scalability.* The formation of $\hat{\Delta}$ requires sorting each of the $n$ rows, which costs $\mathcal{O}(n^2 \log n)$ time. The subsequent linear scans are $\mathcal{O}(n^2)$, so the overall runtime is $\mathcal{O}(n^2 \log n)$. This is slightly slower than computing the ASW, which takes $\mathcal{O}(n^2)$ time [7]. Since the algorithm operates on the entire $n \times n$ distance matrix, the space complexity of Algorithm 1 is $\mathcal{O}(n^2)$. This quadratic memory requirement quickly becomes the dominant bottleneck. For example, a dataset with $10^5$ samples produces $10^{10}$ pairwise distances; storing these in `float32` requires 40 GB of memory. In practice, this limits the algorithm to datasets with at most a few tens of thousands of points on standard hardware. There are ways to mitigate this issue in large-scale settings. For example, one may approximate the upper bound by computing it on a random subset of the samples. Another possibility is to compute and sort distances on the fly rather than materializing the entire matrix. A more thorough investigation of such scalable variants is an interesting direction for future work, but it lies beyond the scope of this paper.

*Output summaries.* Given $\Delta$ and $m = 1$, let $\mathtt{UB}(\Delta)$ denote the mean of the output array $S$ from Algorithm 1. Then $\mathtt{UB}(\Delta)$ corresponds to the right-hand side of Eq. (1). We refer to this as the *global upper bound*. Choosing $m > 1$ imposes the restriction that no cluster is allowed to contain fewer than $m$ objects. Let $\mathtt{UB}_m(\Delta)$ denote the ASW upper bound for this constrained solution space. We will refer to $\mathtt{UB}_m(\Delta)$ as a *constrained upper bound*.

### 4.2. Extension to the macro-averaged silhouette

The standard ASW can produce misleading results when the clusters vary greatly in size. John Pavlopoulos, Georgios Vardakas, and Aristidis Likas [6] show that the so-called macro-averaged silhouette is more robust to cluster size imbalance. For clustering $\mathcal{C}_K = \{C_I\}_{I=1}^K$, the macro-averaged silhouette is defined as

$$S_{macro}(\mathcal{C}_K, \Delta) := \frac{1}{K}\sum_{I=1}^{K}\frac{1}{|C_I|}\sum_{i \in C_I} s(i).$$

In practical scenarios where clusters are imbalanced and small groups also matter, this aggregation strategy is more reasonable. In this subsection, we use Proposition 1 to construct an upper bound for this silhouette variant in a constrained solution space.

Consider an empirical clustering result $\hat{\mathcal{C}}_K = \{\hat{C}_I\}_{I=1}^K$. We narrow our scope in this context to include all clusterings whose cluster sizes match those of $\hat{\mathcal{C}}_K$. This narrower solution space is still generally astronomically large. Thus, an efficiently computable data-dependent upper bound is of interest as a relevant benchmark. Let $\mathcal{C}_K = \{C_I\}_{I=1}^K$ be any clustering from our constrained solution space, and let $\mathtt{ub}_i$ denote the silhouette width upper bound for data point $i$, i.e., $s(i|\mathcal{C}_K, \Delta) \leq \mathtt{ub}_i$. Moreover, let $\{\mathtt{ub}_{\sigma(i)}\}_{i=1}^n$ denote the upper bounds sorted in ascending order, and let $\{N_I\}_{I=1}^K$ denote the assumed cluster sizes in descending order. With $N_0 := 1$, it follows from the rearrangement inequality that

$$\begin{aligned} S_{macro}(\mathcal{C}_K, \Delta) &\leq \frac{1}{K}\sum_{I=1}^{K}\sum_{i \in C_I}\frac{1}{|C_I|}\mathtt{ub}_i \\ &\leq \frac{1}{K}\sum_{I=1}^{K}\frac{1}{N_I}\sum_{i=\sum_{J=0}^{I-1}N_J}^{\sum_{J=1}^{I}N_J}\mathtt{ub}_{\sigma(i)}. \end{aligned}$$

The last expression is independent of any particular cluster assignment and therefore provides a valid upper bound on the macro-averaged silhouette across all clusterings within the constrained solution space. This formulation is particularly useful in applications where cluster sizes are known or can be reasonably estimated, for example in balanced clustering [22]. Even in the absence of such assumptions, the resulting bound serves as a sharper and more informative benchmark for evaluating empirical clustering results.

Because the macro-averaged silhouette weights clusters equally, it carries the danger of favoring solutions with a few small outlier clusters while masking poor structure in the remaining parts. It should therefore be used only when equal importance of all clusters is appropriate. Moreover, we do not recommend routine comparisons with ASW via heuristics such as $\max(\mathrm{ASW}, S_{macro})$. Informed choices should instead be made on a case-by-case basis.

We stress that this extension is not a central focus of the paper. Our aim here is only to illustrate an idea, rather than to conduct a full empirical study of this variant.

## 5. Experiments

We evaluate the ASW upper bound by applying it to a variety of datasets, both synthetic and real. Multiple distance metrics are considered to examine the behavior of the upper bound under different notions of dissimilarity. The code used for the experiments is available in the GitHub repository https://github.com/hugo-strang/silhouette-upper-bound which contains all preprocessing steps, dissimilarity constructions, and scripts required to reproduce the reported results.

For a given dissimilarity matrix, an obvious desideratum of the proposed upper bound is that it is relatively close to the maximum ASW. However, the maximum ASW is of course typically unknown. This presents a fundamental challenge in the evaluation process. To remedy this, we must try to establish a nearly tight lower bound for $\mathrm{ASW}(\mathcal{C}^*, \Delta)$. To this end, we use PAMSIL clustering by default to obtain an empirical ASW. The closer this value is to the upper bound, the narrower the interval in which the true global maximum lies. PAMSIL is a suitable algorithm since its objective is to maximize ASW. When dealing with larger datasets, we use FastMSC and DynMSC instead.

Given an empirical ASW value, say $\hat{S} > 0$, the performance metric of interest is the *relative error*, defined as

$$\frac{\mathrm{ASW}(\mathcal{C}^*, \Delta) - \hat{S}}{\mathrm{ASW}(\mathcal{C}^*, \Delta)} = 1 - \frac{\hat{S}}{\mathrm{ASW}(\mathcal{C}^*, \Delta)},$$

where $\mathcal{C}^*$ denotes a clustering that achieves the global maximum. Since $\mathrm{ASW}(\mathcal{C}^*, \Delta)$ is unknown in practice, we use the global upper bound $\mathtt{UB}(\Delta)$ to obtain a conservative estimate:

$$1 - \frac{\hat{S}}{\mathrm{ASW}(\mathcal{C}^*, \Delta)} \leq 1 - \frac{\hat{S}}{\mathtt{UB}(\Delta)}.$$

We refer to this quantity as the *worst-case relative error*. Its interpretation is straightforward: if the worst-case relative error falls below a tolerance $\epsilon$, then the empirical ASW is guaranteed to be within $\epsilon$ of the true maximum. Instead of employing this measure to compare clustering algorithms, we will use it to evaluate the tightness of the upper bound relative to the ASW obtained by the chosen algorithm.

The results of both the global upper bound and the constrained upper bound are reported. In our experiments, we do not impose external constraints on the minimum cluster size. Instead, for each empirical clustering solution (e.g., produced by PAMSIL), we simply record its smallest cluster size and compute $\mathtt{UB}_m(\Delta)$ using this value of $m$. This approach does not enforce a constraint but allows us to observe how the constrained upper bound behaves under the cluster size patterns that arise in practice.

**Table 2**
**Experimental results on synthetic data.** Comparison of clustering performance and the proposed upper bound. Avg. distance indicates the average pairwise distance between cluster centers.

| Dataset | $K$ | ARI | AMI | ASW | $\mathtt{UB}(\Delta)$ | $1-\frac{\text{ASW}}{\mathtt{UB}(\Delta)}$ | Avg. distance |
|---|---|---|---|---|---|---|---|
| 400-64-2-2 | 2 | 1.0 | 1.0 | 0.673 | 0.673 | 0.000 | 64.5 |
| 400-64-5-6 | 5 | 1.0 | 1.0 | 0.249 | 0.376 | 0.339 | 65.2 |
| 400-128-7-3 | 7 | 1.0 | 1.0 | 0.522 | 0.566 | 0.078 | 92.3 |
| 1000-300-5-2 | 5 | 1.0 | 1.0 | 0.668 | 0.678 | 0.015 | 143.7 |
| 10000-32-20-2 | 20 | 1.0 | 1.0 | 0.626 | 0.774 | 0.192 | 47.0 |
| 10000-1024-20-4 | 20 | 1.0 | 1.0 | 0.417 | 0.454 | 0.081 | 263.1 |

### 5.1. Synthetic data

In this section, we look at the proposed upper bound applied to synthetic data generated by the `make_blobs` function from the scikit-learn library [23]. In this idealized setting, we also demonstrate why a complete scan of the $k$-values is necessary when constructing $f(i,\Delta)$ in Proposition 1. Indeed, it is reasonable to expect that $k=2$ often minimizes $q(i,\Delta,k)$. Lastly, we use synthetic dissimilarity matrices to evaluate the runtime of Algorithm 1. In the following sections, we evaluate the usefulness of the upper bound in more realistic settings.

For dataset notation, we follow [24]. Each dataset is labeled as `n_samples-n_features-centers-cluster_std`, indicating the parameter values used to generate the points. For example, 400-16-3-4 denotes a dataset of 400 16-dimensional vectors partitioned into 3 clusters with standard deviation 4, which controls within-cluster dispersion. For these datasets, we use Euclidean distance.

In Table 2, we apply the upper bound to six synthetic datasets and compare the results with PAMSIL clustering. For dataset 400-64-2-2, the upper bound at 0.673 confirms that PAMSIL reaches a globally optimal silhouette score. For three of the remaining datasets, $\mathtt{UB}(\Delta)$ proves that the PAMSIL ASW is within 8 % of the global maximum. For dataset 400-64-5-6 with $\mathtt{UB}(\Delta)=0.376$, we record an empirical ASW of 0.249. In this case, it remains unclear whether further substantial improvements of the ASW are possible. However, interpreting the attained value in the interval $[-1, 0.376]$ enables more well-founded conclusions compared to using the standard range $[-1,1]$. Lastly, the reported ARI and AMI scores tell us that PAMSIL finds the ground truth clustering in all six datasets.

Fig. 3 illustrates why it is necessary to scan all values of $k$ when constructing the upper bound in Proposition 1. In this example, we have picked a random data point $i$ from the dataset 1000-300-5-2, and the graph shows the quantity $1-q(i,\Delta,k)$ over a range of $k$-values, demonstrating that, for example, $k=2$ does not correspond to a valid bound in this case. Instead, the minimum of $q(i,\Delta,k)$ occurs at $k=200$, which coincides with the generated cluster sizes. Although in real-world datasets with less clearly separated clusters, we generally expect the difference between the average distance to the $n-2$ farthest neighbors and the nearest-neighbor distance to be sufficiently large that $k=2$ minimizes $q$. We suspect that this behavior often holds true even for data points whose cluster size in $C^*$ is much larger than 2. Assuming that this speculation is true, it would be very unlikely that $\mathtt{UB}(\Delta)$ is actually close to sharp in real-world settings.

Fig. 4 shows runtime measurements for Algorithm 1 ($m=1$). Additionally, we include the runtime of constructing $\hat{\Delta}$ from $\Delta$, i.e., the second row in Algorithm 1. As a reference, we also report the time needed to compute the dissimilarity matrix $\Delta$ (although $\Delta$ is assumed to be precomputed in Algorithm 1). The experiments were conducted on an Apple MacBook Pro-(M1, 2020) with 16GB of RAM, and the software environment was Python 3.10. Distance matrices of varying size and `float32` data type were generated through the NumPy and scikit-learn packages. Specifically, we used the `random.randn` function to generate $n$ samples with 32 features (seed fixed to 0), and the `metrics.pairwise_distances` function to generate all pairwise distances. The runtime of Algorithm 1 does not depend on the degree of clustering structure present in the dataset, which is why `make_blobs` was not used for this task. For each distance matrix, the programs were executed five times. The graph shows the average runtimes for each number of samples $n$, and the error bars indicate one standard deviation. It is clear from the figure that in our implementation, a substantial amount of the runtime of Algorithm 1 is spent on sorting each of the $n$ rows of $\Delta$ to obtain $\hat{\Delta}$.

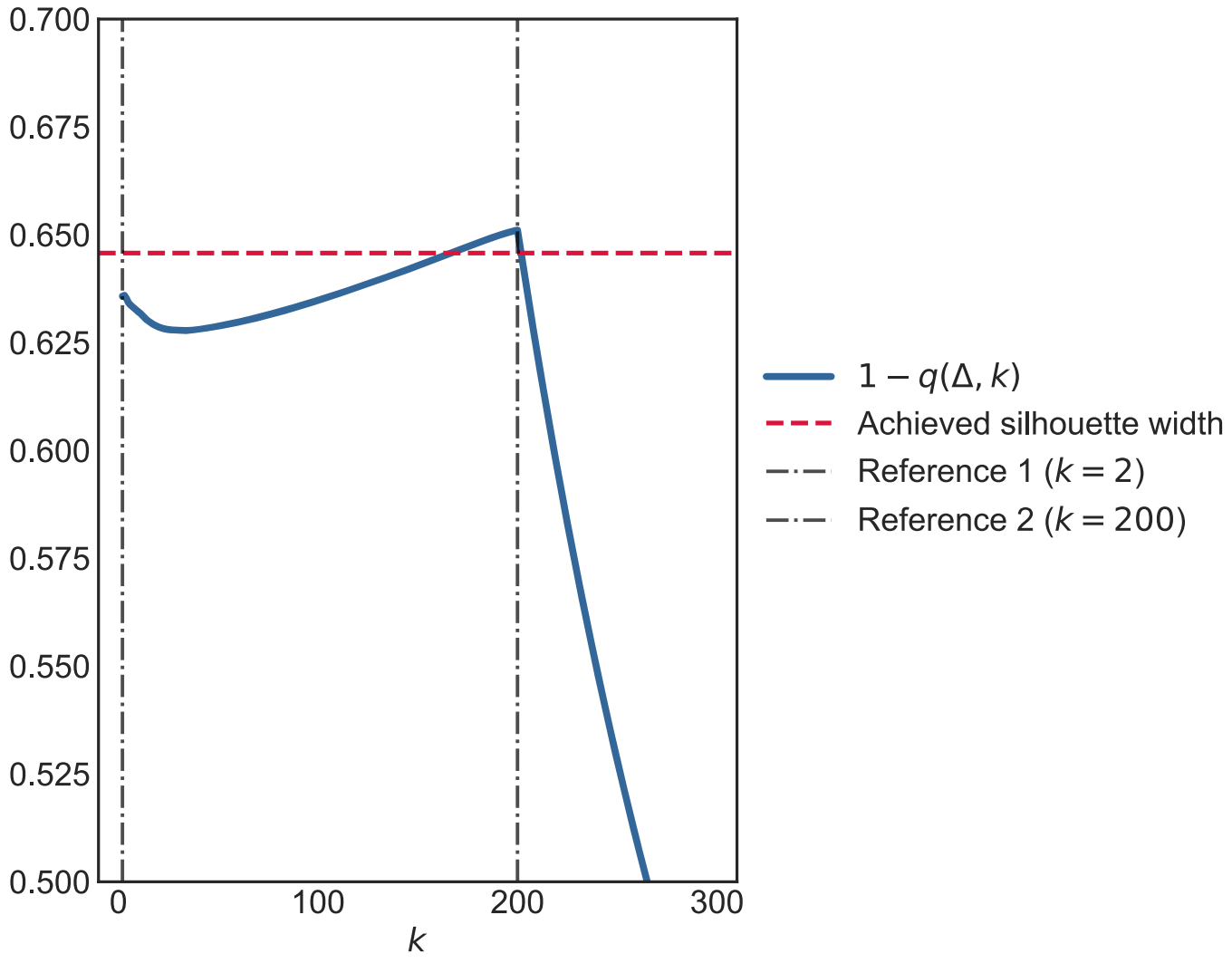

**Fig. 3.** For a random data point in the synthetic dataset 1000-300-5-2, we show the quantity $1-q(\Delta,k)$ for a range of $k$-values. The red dashed horizontal line shows the point's silhouette width achieved by PAMSIL clustering with $K=5$. The maximum of $1-q(\Delta,k)$ is the upper bound from Proposition 1. The maximum is given at $k=200$, which corresponds to the cluster sizes in this dataset. (For interpretation of the references to colour in this figure legend, the reader is referred to the web version of this article.)

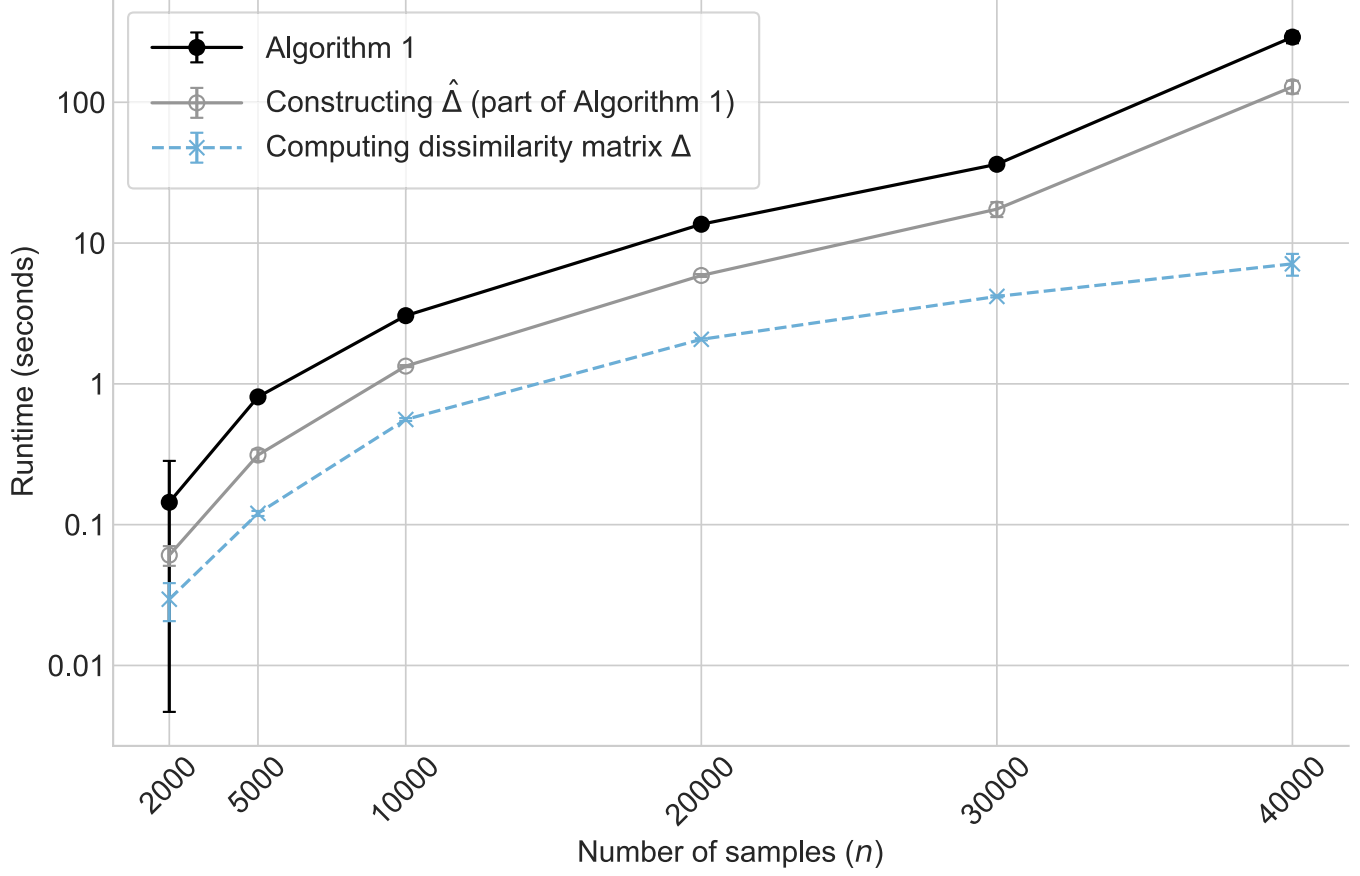

**Fig. 4.** Runtime experiments on synthetic data (log-linear plot).

### 5.2. UCI datasets

We evaluate our method on eight real-world datasets from the UCI Machine Learning Repository [25], commonly used in clustering and classification research. Table 3 summarizes their key characteristics along with the distance metrics employed. Standard preprocessing steps, including feature scaling and outlier removal, are applied where appropriate. Each dataset provides a known number of classes, which we use to determine the number of clusters $K$. When all class labels are available, we report the ARI and AMI to verify that these datasets represent meaningful clustering structures.

Table 4 examines the performance of PAMSIL on the eight real-world datasets and compares it with the upper bounds. In these cases, the gap between the empirical ASW and the global upper bound is generally quite large. It is unclear whether the reason for this is that the attained ASW values are far from $\text{ASW}(C^*, \Delta)$, or that the upper bounds are rather loose, or both. However, the constrained upper bounds $\mathtt{UB}_m(\Delta)$ are considerably tighter for most datasets. Here, $m$ is set to the minimum cluster size in the PAMSIL solution. For datasets Ceramic, Customers, Rna, Wdbc, and Wine, $\mathtt{UB}_m(\Delta)$ proves that the ASW-values attained by PAMSIL are within 30 % of the corresponding optimums (within the solution spaces constrained by $m$).

**Table 3**
**UCI datasets.** Datasets used in the experiments and their characteristics. The Metric column indicates the distance metrics used.

| Dataset | Title | Instances | Features | Classes | Field | Metric |
|---|---|---|---|---|---|---|
| Ceramic | Chemical Composition of Ceramic Samples | 88 | 17 | 2 | Physics | Manhattan |
| Customers | Wholesale customers | 378 | 7 | 2 | Business | Cosine |
| Dermatology | Dermatology | 366 | 34 | 6 | Health and Medicine | Cosine |
| Heart Statlog | Statlog (Heart) | 231 | 13 | 2 | Health and Medicine | Euclidean |
| Optdigits | Optical Recognition of Handwritten Digits | 2296 | 64 | 10 | Computer Science | Euclidean |
| Rna | Gene expression cancer RNA-Seq | 808 | 20,531 | 5 | Biology | Euclidean |
| Wdbc | Wisconsin Diagnostic Breast Cancer | 382 | 30 | 2 | Health and Medicine | Euclidean |
| Wine | Wine | 178 | 13 | 3 | Chemistry | Manhattan |

**Table 4**
**Experimental results on UCI datasets.** Comparison of clustering performance and the proposed upper bound.

| Performance of PAMSIL clustering | | | | | | Comparison with global upper bound | | Comparison with constrained upper bound | |
|---|---|---|---|---|---|---|---|---|---|
| Dataset | $K$ | $\min_I\{\lvert C_I\rvert\}$ | ARI | AMI | ASW | $\mathtt{UB}(\Delta)$ | $1-\frac{\text{ASW}}{\mathtt{UB}(\Delta)}$ | $\mathtt{UB}_m(\Delta)$ | $1-\frac{\text{ASW}}{\mathtt{UB}_m(\Delta)}$ |
| Ceramic | 2 | 44 | 1.000 | 1.000 | 0.358 | 0.621 | 0.423 | 0.403 | 0.112 |
| Customers | 2 | 170 | 0.394 | 0.366 | 0.400 | 0.951 | 0.579 | 0.574 | 0.303 |
| Dermatology | 6 | 21 | 0.850 | 0.875 | 0.455 | 0.836 | 0.456 | 0.704 | 0.354 |
| Heart Statlog | 2 | 83 | 0.447 | 0.349 | 0.161 | 0.600 | 0.731 | 0.279 | 0.422 |
| Optdigits | 10 | 50 | 0.803 | 0.819 | 0.180 | 0.620 | 0.709 | 0.463 | 0.611 |
| Rna | 5 | 77 | – | – | 0.225 | 0.419 | 0.464 | 0.323 | 0.303 |
| Wdbc | 2 | 46 | 0.427 | 0.386 | 0.348 | 0.621 | 0.440 | 0.455 | 0.236 |
| Wine | 3 | 53 | 0.884 | 0.864 | 0.315 | 0.649 | 0.514 | 0.450 | 0.298 |

As for the global version, the lowest upper bound we recorded was for the Rna dataset, for which PAMSIL reached an ASW of 0.225 and the upper bound was 0.419. In contrast, the highest upper bound we recorded was for the Customers dataset, for which PAMSIL reached an ASW of 0.4 and the upper bound was 0.951. The smallest and largest worst-case relative errors were for the Ceramic dataset (0.423) and the Heart Statlog dataset (0.731), respectively.

We observe that the gaps between $\mathtt{UB}(\Delta)$ and $\mathtt{UB}_m(\Delta)$ are rather large for most datasets. For instance, the Heart Statlog dataset showed $\mathtt{UB}(\Delta) = 0.6$ and $\mathtt{UB}_{83}(\Delta) = 0.28$ — a 53 % decrease. Earlier in the text, when discussing Fig. 3, we speculated that, in many real-world datasets, the $k$-quotient will frequently reach its minimum at $k = 2$ regardless of the cluster structure in $C^*$. If $q(\Delta, m) \gg q(\Delta, 2)$ this naturally leads to large gaps. On the one hand, the constrained upper bound is an important feature of the proposed method that can lead to significantly sharper insights than those provided by the global upper bound. On the other hand, implementing an unjustified constraint by mistake can cause incorrect conclusions. Therefore, we recommend careful consideration when choosing the parameter $m$.

Fig. 5 shows classical silhouette visualizations corresponding to the PAMSIL clusters, complemented with upper bounds of the individual silhouette widths. The solid contour lines indicate the point-wise silhouette widths, the opaque areas indicate the corresponding constrained upper bounds, while the transparent areas indicate the corresponding global upper bounds. For the constrained upper bounds, $m$ is set to the smallest cluster size in the generated solution. This method provides a visual comparison of the empirical silhouette and the highest possible value at the data point level.

In Fig. 6 we show, for each dataset and for each $K \in \{2, \dots, 15\}$, the attained ASW value compared to the global ASW upper bound. Additionally, at each $K$, we compare the ASW to the constrained upper bound $\mathtt{UB}_m(\Delta)$ with $m$ set to the minimum cluster size in the corresponding solution. Lastly, we include, at each $K$, the recorded macro-averaged silhouette and compare this value with the upper bound derived in Section 4.2. This scan over different $K$ reveals that PAMSIL reaches a higher or equally high ASW at a value of $K$ that is different than the ground truth number of classes for all datasets except Ceramic and Wdbc.

### 5.3. *The ALOI dataset*

The purpose of this experiment is to study how the proposed upper bound behaves on larger datasets with a large number of clusters. In particular, we investigate whether the bound appears informative when the number of classes is large.

We base our experiments on data derived from the Amsterdam Library of Object Images (ALOI) [26], a large-scale image collection containing objects photographed under varying viewing angles and illumination conditions. Each image is represented using HSV/HSB color histograms with different binning resolutions. In this representation, each pixel is transformed into hue-saturation-brightness space and assigned to its nearest histogram bin. We consider two such representations, namely 5×5×5 and 14×6×6 binnings, resulting in feature vectors of up to several hundred dimensions. In this setting, we use the Manhattan distance metric.

The datasets have 1000 object classes with at least 100 instances available for each class. To obtain a total number of samples that is manageable, we randomly select 40 instances per class, yielding datasets with 40,000 samples in total. In this setting, PAMSIL becomes computationally prohibitive, and we therefore rely on FastMSC and DynMSC, which are significantly more scalable while still directly optimizing silhouette-based objectives. For these experiments, we report both clustering runtimes and the runtimes required to compute the global and constrained upper bounds. Table 5 summarizes the results.

With $K = 1000$, the gap between the global upper bound and the empirical ASW is rather substantial for both datasets. However, applying DynMSC ($2 \le K \le 50$) reveals that a much higher ASW is achievable at $K = 2$.

For the 14×6×6 dataset with $m = 2$ and $m = 3$, the constrained upper bounds show slightly higher runtimes than the global upper bounds, which we attribute to runtime variability in the experimental environment.

We also conduct some small-scale experiments with 2 and 5 classes from the same underlying data. For these datasets, we sample 100 in-stances from each class to obtain datasets of sizes 200 and 500, respectively. For instance, 5×5×5–2 and 5×5×5–5 are datasets derived from the same image representations, but with different number of classes (2 and 5). Since these datasets are so small, we can apply PAMSIL and compare it with FastMSC and DynMSC. Table 6 summarizes the results.

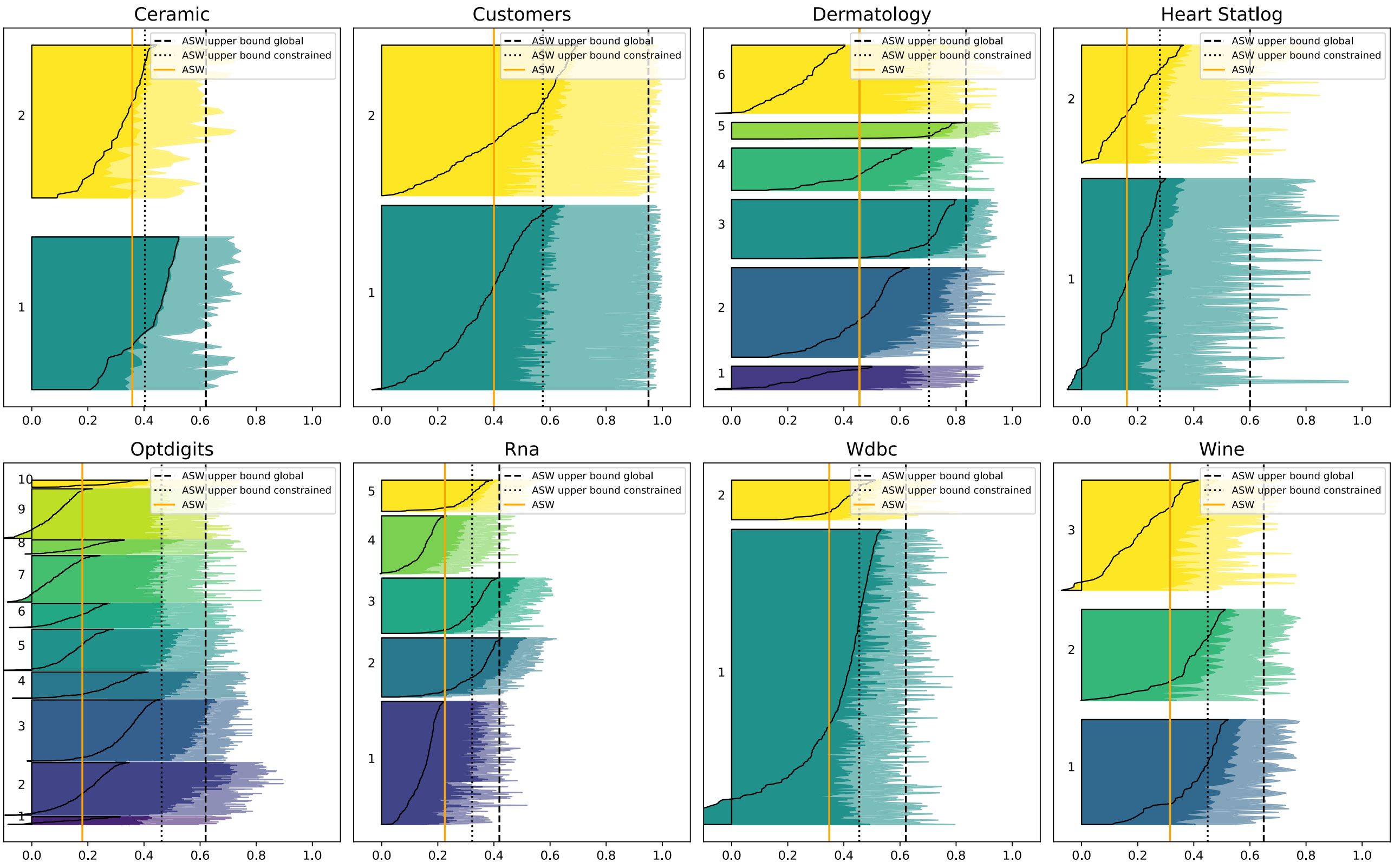

**Fig. 5.** Silhouette plots for the UCI datasets. PAMSIL clustering with $K$ equal to the ground truth number of classes; ASW compared against the global upper $\mathtt{UB}(\Delta)$ and the constrained upper bound $\mathtt{UB}_m(\Delta)$, with $m$ being the smallest cluster size. The solid contour lines indicate the point-wise silhouette widths, the opaque areas indicate the corresponding constrained upper bounds, while the transparent areas indicate the corresponding global upper bounds.

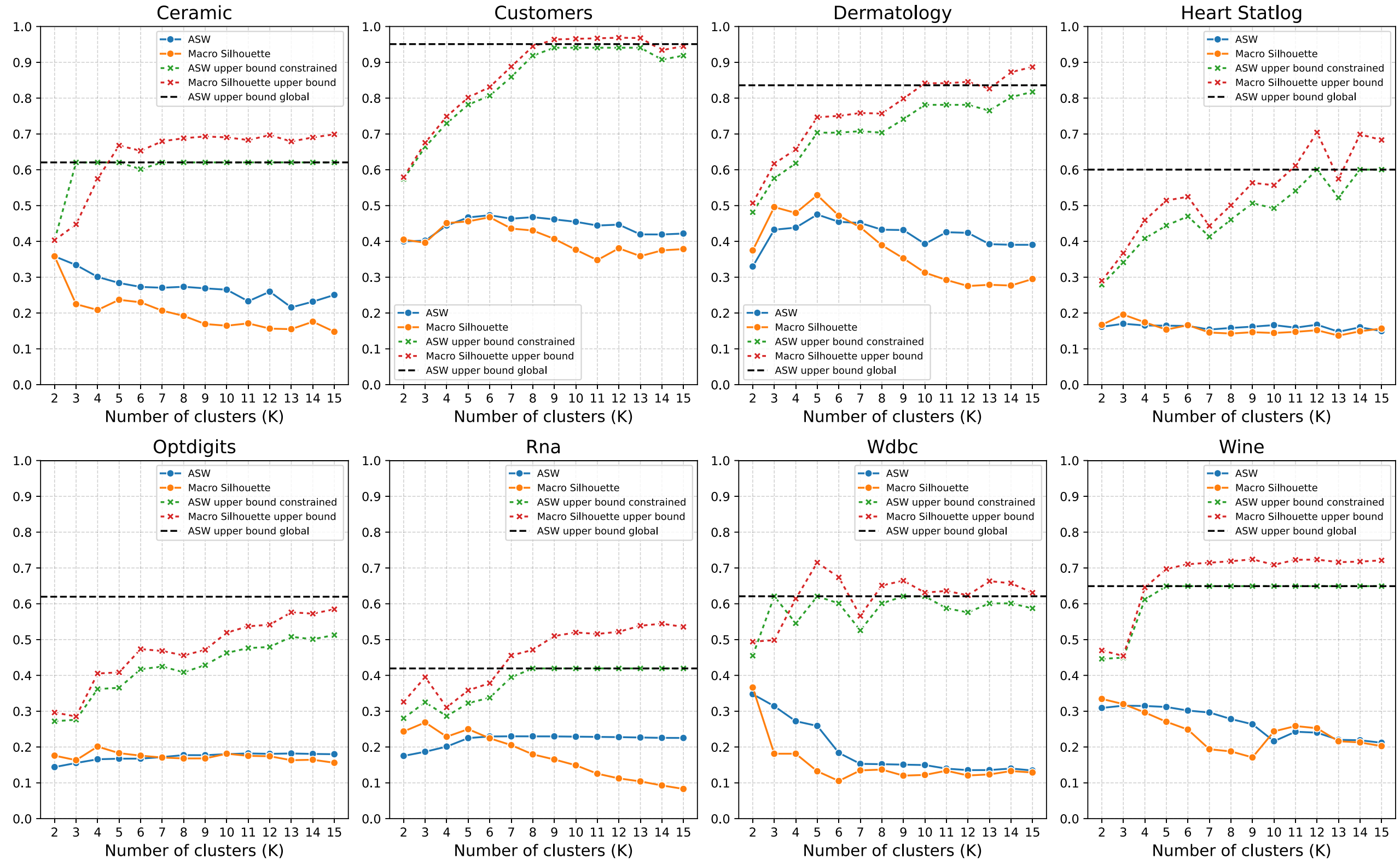

**Fig. 6.** PAMSIL performance for varying $K$ and comparison with upper bounds.

**Table 5**
**Experimental results on ALOI datasets with 1000 classes.** Comparison of clustering performance and the proposed upper bound. DynMSC scans $2 \le K \le 50$ and FastMSC$^{\dagger}$ scans $K = 100, 200, \ldots, 900$. Runtime (RT) is expressed in 1000 s.

| Performance of FastMSC and DynMSC clustering | | | | | | | | Comparison with global upper bound | | | Comparison with constrained upper bound | | |
|---|---|---|---|---|---|---|---|---|---|---|---|---|---|
| Dataset | Algorithm | $K$ | $\min_I\{\lvert C_I\rvert\}$ | ARI | AMI | ASW | RT | $\mathtt{UB}(\Delta)$ | $1-\frac{\text{ASW}}{\mathtt{UB}(\Delta)}$ | RT | $\mathtt{UB}_m(\Delta)$ | $1-\frac{\text{ASW}}{\mathtt{UB}_m(\Delta)}$ | RT |
| 5×5×5 | FastMSC | 1000 | 2 | 0.267 | 0.557 | 0.133 | 1.1 | 0.914 | 0.854 | 0.3 | 0.913 | 0.854 | 0.3 |
| | FastMSC$^{\dagger}$ | 100 | 7 | 0.006 | 0.368 | 0.181 | 9.7 | 0.914 | 0.802 | 0.3 | 0.866 | 0.791 | 0.3 |
| | DynMSC | 2 | 976 | 0.000 | 0.009 | 0.502 | 1.6 | 0.914 | 0.451 | 0.3 | 0.618 | 0.188 | 0.3 |
| 14×6×6 | FastMSC | 1000 | 2 | 0.292 | 0.589 | 0.142 | 1.1 | 0.901 | 0.842 | 0.2 | 0.900 | 0.842 | 0.3 |
| | FastMSC$^{\dagger}$ | 100 | 3 | 0.005 | 0.352 | 0.161 | 9.9 | 0.901 | 0.822 | 0.2 | 0.883 | 0.818 | 0.3 |
| | DynMSC | 2 | 38 | 0.000 | 0.001 | 0.470 | 1.7 | 0.901 | 0.478 | 0.2 | 0.752 | 0.375 | 0.2 |

**Table 6**
**Experimental results on ALOI datasets with a small number of classes.** Comparison of clustering performance and the proposed upper bound. DynMSC scans $2 \le K \le 10$.

| Performance of PAMSIL clustering | | | | | | | Comparison with global upper bound | | Comparison with constrained upper bound | |
|---|---|---|---|---|---|---|---|---|---|---|
| Dataset | Algorithm | $K$ | $\min_I\{\lvert C_I\rvert\}$ | ARI | AMI | ASW | $\mathtt{UB}(\Delta)$ | $1-\frac{\text{ASW}}{\mathtt{UB}(\Delta)}$ | $\mathtt{UB}_m(\Delta)$ | $1-\frac{\text{ASW}}{\mathtt{UB}_m(\Delta)}$ |
| 5×5×5–2 | PAMSIL | 2 | 94 | 0.883 | 0.834 | 0.669 | 0.899 | 0.256 | 0.715 | 0.064 |
| | FastMSC | 2 | 94 | 0.883 | 0.834 | 0.669 | 0.899 | 0.256 | 0.715 | 0.064 |
| | DynMSC | 8 | 1 | 0.876 | 0.820 | 0.647 | 0.899 | 0.281 | 0.899 | 0.281 |
| 5×5×5–5 | PAMSIL | 5 | 24 | 0.675 | 0.751 | 0.478 | 0.899 | 0.468 | 0.762 | 0.373 |
| | FastMSC | 5 | 19 | 0.654 | 0.723 | 0.473 | 0.899 | 0.473 | 0.778 | 0.392 |
| | DynMSC | 2 | 28 | 0.001 | 0.009 | 0.416 | 0.899 | 0.537 | 0.752 | 0.446 |
| 14×6×6–2 | PAMSIL | 2 | 94 | 0.883 | 0.834 | 0.640 | 0.879 | 0.271 | 0.695 | 0.078 |
| | FastMSC | 2 | 92 | 0.846 | 0.795 | 0.639 | 0.879 | 0.273 | 0.702 | 0.090 |
| | DynMSC | 9 | 1 | 0.856 | 0.776 | 0.644 | 0.879 | 0.267 | 0.879 | 0.267 |
| 14×6×6–5 | PAMSIL | 5 | 24 | 0.670 | 0.739 | 0.473 | 0.888 | 0.467 | 0.752 | 0.371 |
| | FastMSC | 5 | 23 | 0.676 | 0.754 | 0.471 | 0.888 | 0.470 | 0.755 | 0.376 |
| | DynMSC | 6 | 20 | 0.678 | 0.746 | 0.472 | 0.888 | 0.468 | 0.764 | 0.382 |

It is difficult to draw strong conclusions from Table 5. The large-scale setting introduces substantial uncertainty regarding the most appropriate number of clusters, which complicates a precise assessment of the bounds tightness.

In contrast, Table 6 provides a more controlled comparison. The results suggest that the gap between the maximum ASW and the proposed upper bound increases as the number of clusters in the desirable solution grows, even when moving from just 2 to 5 clusters. This observation highlights an important limitation of the method: the bound appears to be most informative when the underlying clustering structure involves a relatively small number of clusters. Consequently, the proposed upper bound should be interpreted with great caution in settings where the true number of clusters is large, and its practical usefulness is expected to be greatest when a small number of clusters is assumed.

### *5.4. Discussion*

In Table 2, we saw that $\mathtt{UB}(\Delta)$ was able to confirm that PAMSIL reached a globally optimal ASW for the synthetic dataset 400-64-2-2. This is interesting from a mathematical perspective, but it is not a realistic use case. For any real-world dataset, we expect the likelihood of the bound being sharp to be rather low.

Replacing the standard ASW range $[-1, 1]$ with the data-dependent interval $[-1, \mathtt{UB}(\Delta)]$ enhances the interpretability of observed ASW values. However, when the empirical ASW lies far below the upper bound, it remains unclear to what extent further improvement is actually possible. We suspect this gap to often be substantial for the global upper bound, even when the empirical ASW is close to $\text{ASW}(C^*, \Delta)$. In many real-world applications, however, analysts impose minimum cluster size constraints to avoid overly small or degenerate clusters [20,21]. By incorporating such constraints through $m$, the refined bound $\mathtt{UB}_m(\Delta)$ better reflects the relevant solution space and provides more meaningful guidance in practice. As demonstrated in Table 4, $\mathtt{UB}_m(\Delta)$ ensures that the PAMSIL solution is within 30 % of the optimum (in the constrained solution space) for five of the examined datasets.

The experimental results indicate that the usefulness of the upper bound varies between different contexts; understanding when it is most beneficial in practice is important. For the sake of argument, we can assume that the bound is useful whenever it is far from 1.0 (say, less than 0.5). Following this logic, one can expect the bound to be most informative when the dimensionality of the data is rather high. The curse of dimensionality tells us that distances become more similar when dimensionality increases, which should result in the upper bound getting closer to zero. Indeed, for the real-world cases, the smallest upper bound we observed was for the Rna dataset with Euclidean distances and 20,531 features. In contrast, the largest upper bound we observed was for the Customers dataset (Cosine distance), with only 7 features. In addition to the fact that the data should be high-dimensional, we speculate that the upper bound is more useful when the number of clusters in the optimal solution is relatively low.

## 6. Summary

### *6.1. Limitations*

*Challenges with the ASW.* A natural limitation of the proposed method stems from the limitations of the silhouette metric itself. Many clustering quality indices have been proposed over time, and the ASW is not inherently superior to these. Its usefulness depends on the structure of the data. For instance, the ASW can be misleading when the underlying clusters are of different diameters or strongly anisotropic, in which case a desirable partition can even result in negative silhouette widths. Also, singletons are usually given a silhouette width of 0 by default, even when they are well separated. In conclusion, the proposed upper bound should only be considered in situations where the silhouette is meaningful.

*Non-sharpness.* The proposed upper bound is not guaranteed to be close to the maximum ASW. Therefore, the upper bound should not be viewed as a precise target, but rather an informative ceiling. Generally, we expect the bound to be rather loose when the number of underlying clusters is large.

*Challenges with evaluation.* A fundamental limitation inherent to the problem is that the true global maximum of the ASW is unknown for any non-trivial dataset. Consequently, the quality of our upper bound cannot be assessed exactly, but only indirectly through empirical evidence. This makes it difficult to draw precise conclusions.

*Poor scalability.* Since the proposed algorithm assumes the full dissimilarity matrix as input, it suffers from the same scalability issue as that for computing the ASW, namely that $\mathcal{O}(n^2)$ memory is required. In principle, one could design an implementation that instead computes and sorts distances per data point, reducing the memory requirement to $\mathcal{O}(n)$. However, this does not change the overall time complexity, which remains $\mathcal{O}(n^2 \log n)$. Therefore, the proposed method does not scale to very large datasets.

*Hardware constraints.* Due to hardware constraints, large datasets (e.g., $n > 50000$) were not included in our empirical evaluation. This does not affect the theoretical correctness of the method, but limits our empirical assessment at large scales.

### 6.2. Conclusion

In this paper, we proposed a dataset-dependent upper bound on the ASW as a complementary tool for analyzing clustering quality. While ASW is conventionally interpreted relative to its standard range $[-1, 1]$, our results show that a tighter data-dependent upper bound can be computed in $\mathcal{O}(n^2 \log n)$ time. This upper bound provides a more relevant reference value. It often lies well below 1, reflecting the inherent limitations of the dataset itself.

As for the practical usefulness of the proposed upper bound, our findings do not indicate that meaningful insights can be expected in the general case. However, it is evident from the results that there are contexts in which the upper bound succeeds in providing useful information. For example, for the Rna dataset, we observed an upper bound of 0.42, which placed the empirical ASW of 0.23 in a completely different lighting. Pinpointing these cases where the upper bound is the most beneficial would be of great value.

From a broader perspective, our main contribution has been to establish a proof of concept: that efficiently computable, data-dependent upper bounds on ASW can enrich cluster analysis. We believe that this idea opens a promising direction for future work, where the focus should be on deriving tighter bounds and extending the framework to other internal clustering validity indices, for example, to the medoid-based silhouette. An important open theoretical question, which we do not address here, is to characterize how loose the proposed upper bound is in a given case, for example, by deriving a lower bound on the gap between the true maximum ASW and $\mathtt{UB}(\Delta)$. As mentioned above, another important next step is to better understand under which conditions the proposed bound is most informative in practice — for instance, how its usefulness varies with properties of the data, the dissimilarity metric, the dimensionality of the feature space, or the number of underlying clusters. Such insights would clarify in which contexts the bound provides the greatest practical value.

## CRediT authorship contribution statement

**Hugo Sträng:** Conceptualization, Formal analysis, Investigation, Methodology, Visualization, Writing - original draft, Writing - review & editing; **Tai Dinh:** Validation, Visualization, Writing - review & editing.

## Data availability

The source code, libraries, and datasets are open-source and available on GitHub, with the access link provided in the paper.

## Declaration of competing interest

The authors declare that they have no known competing financial interests or personal relationships that could have appeared to influence the work reported in this paper.

## Acknowledgments

The authors would like to thank Hugo Fernström for his careful proofreading and thoughtful feedback on an earlier version of this manuscript. His continuous encouragement has been greatly appreciated throughout the development of this work.

We are also grateful to several researchers for their valuable comments and suggestions on a previous version of the paper.

T. Dinh is also affiliated with The Kyoto College of Graduate Studies for Informatics (KCGI), Kyoto, Japan. This work was carried out while he was a visiting researcher at Thu Dau Mot University.